\documentclass[letterpaper]{article} 
\usepackage{aaai2026}  
\usepackage{times}  
\usepackage{helvet}  
\usepackage{courier}  
\usepackage[hyphens]{url}  
\usepackage{graphicx} 
\usepackage{natbib}  
\usepackage{caption} 
\usepackage{algorithm}
\usepackage{algorithmic}
\usepackage{amsthm}
\newtheorem{theorem}{Theorem}
\usepackage{algorithm}
\usepackage{algorithmic}
\usepackage{amsmath} 
\usepackage{booktabs}
\usepackage{multirow}
\usepackage{amsthm} 
\usepackage{appendix}

\usepackage{bm}        
\usepackage{amssymb}   
\usepackage{subcaption}
\usepackage{newfloat}
\usepackage{listings}
\DeclareCaptionStyle{ruled}{labelfont=normalfont,labelsep=colon,strut=off} 
\floatstyle{ruled}
\newfloat{listing}{tb}{lst}{}
\floatname{listing}{Listing}
\title{Certified L2-Norm Robustness of 3D Point Cloud Recognition in the Frequency Domain}
\author{
    Liang Zhou\textsuperscript{\rm 1},
    Qiming Wang\textsuperscript{\rm 1}\thanks{Corresponding author.},
    Tianze Chen\textsuperscript{\rm 1}
}
\affiliations{
    \textsuperscript{\rm 1}School of Mathematics, Hohai University, China\\
    zeallll1127@gmail.com, wqm@hhu.edu.cn, 241319020003@hhu.edu.cn
}

\usepackage{bibentry}

\begin{document}

\maketitle

\begin{abstract}
3D point cloud classification is a fundamental task in safety-critical applications such as autonomous driving, robotics, and augmented reality. However, recent studies reveal that point cloud classifiers are vulnerable to structured adversarial perturbations and geometric corruptions, posing risks to their deployment in safety-critical scenarios. Existing certified defenses limit point-wise perturbations but overlook subtle geometric distortions that preserve individual points yet alter the overall structure, potentially leading to misclassification. In this work, we propose \textbf{FreqCert}, a novel certification framework that departs from conventional spatial domain defenses by shifting robustness analysis to the frequency domain, enabling structured certification against global $\ell_2$-bounded perturbations. FreqCert first transforms the input point cloud via the graph Fourier transform (GFT), then applies structured frequency-aware subsampling to generate multiple sub-point clouds. Each sub-cloud is independently classified by a standard model, and the final prediction is obtained through majority voting, where sub-clouds are constructed based on spectral similarity rather than spatial proximity, making the partitioning more stable under $\ell_2$ perturbations and better aligned with the object’s intrinsic structure. We derive a closed-form lower bound on the certified $\ell_2$ robustness radius and prove its tightness under minimal and interpretable assumptions, establishing a theoretical foundation for frequency domain certification. Extensive experiments on the ModelNet40 and ScanObjectNN datasets demonstrate that FreqCert consistently achieves higher certified accuracy and empirical accuracy under strong perturbations. Our results suggest that spectral representations provide an effective pathway toward certifiable robustness in 3D point cloud recognition.
\end{abstract}


\section{Introduction}

3D point clouds, which represent unordered sets of spatial coordinates, are widely adopted in safety-critical applications such as autonomous driving, robotics, and industrial inspection \cite{9943630, 9901642, 10425012, 10312684}. Due to their sparse, irregular, and non-Euclidean structure, point clouds pose unique challenges for perception tasks, prompting the development of dedicated deep learning architectures\cite{8099499,10.5555/3295222.3295263,li_pointcnn_2018,wang_dynamic_2019,9010002,zhao2021point}.

While these models achieve high classification accuracy under clean conditions, recent studies have revealed their alarming vulnerability to adversarial attacks—small, imperceptible perturbations in the point cloud can easily cause models to make erroneous predictions \cite{xiang_generating_2019,liu2019extending}. These adversarial manipulations, in addition to natural corruptions such as occlusions and noise, pose serious security risks due to their subtlety and targeted attack nature \cite{tsai2020robust,wen2020geometry}. 

To address this, a variety of empirical defense strategies have been proposed. Typical methods involve input transformation or purification strategies, such as DUP-Net which combines outlier removal and upsampling to defend against adversarial perturbations \cite{zhou_dup-net_2019}, IF-Defense which utilizes implicit surfaces for input projection \cite{wu_if-defense_2021}, and diffusion-based purification such as PointDP \cite{sun_critical_2023}. Despite improving empirical robustness, these methods lack certified guarantees, leaving them susceptible to adaptive attacks \cite{perez20223deformrs,lorenz2021robustness,li2022robust}. Recent test-time adaptation methods, such as Purified Self-Training (PST) \cite{lin_improving_2024}, attempt to improve robustness during inference, but still fail to provide provable robustness certificates in most practical scenarios.

\begin{figure*}[t]
\centering
\includegraphics[width=0.85\textwidth]{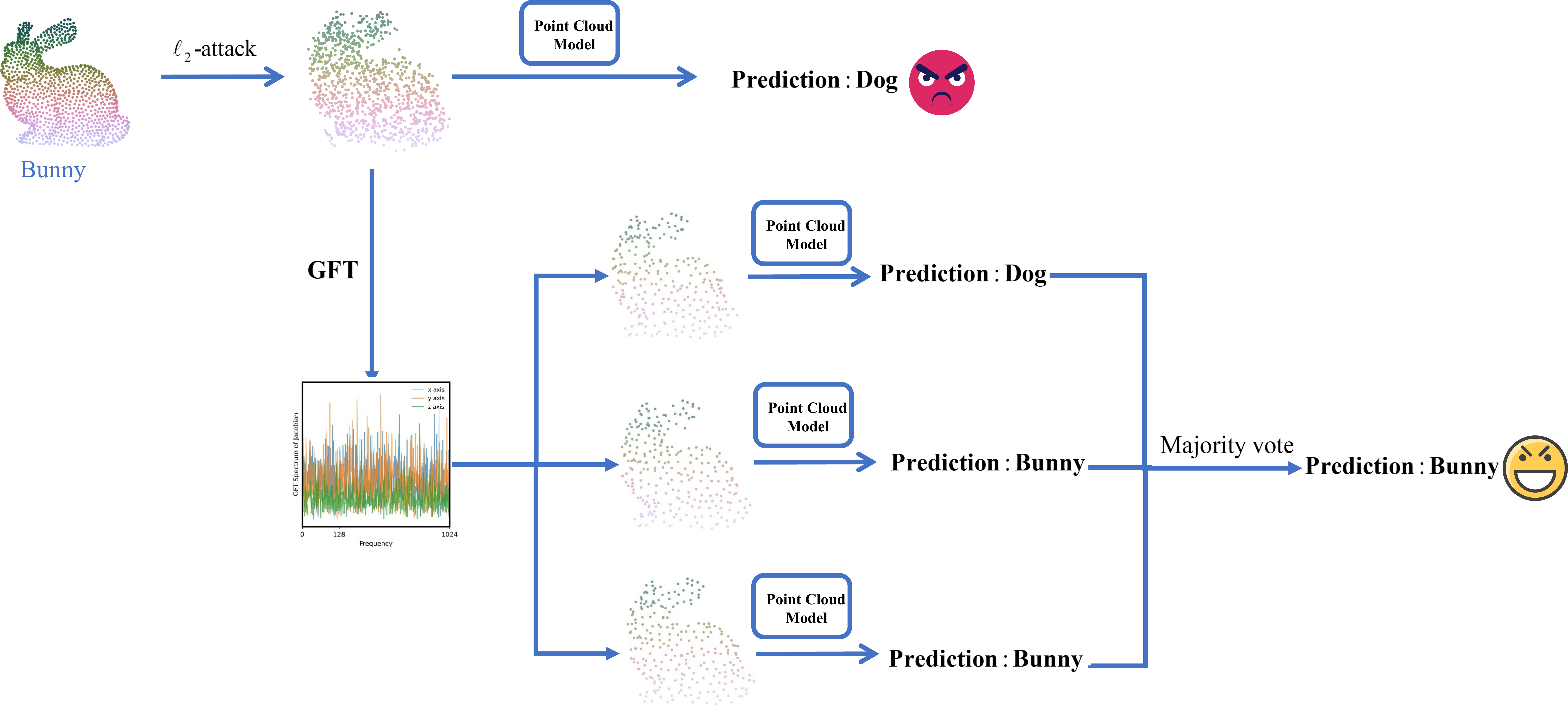}
\caption{Illustration of our framework. The input point cloud is processed into spectral slices via graph Fourier transform (GFT), each of which is independently classified. The final prediction is obtained by majority voting.}
\label{fig:framework}
\end{figure*}

To bridge this gap, certified defenses have been proposed to provide formal robustness guarantees.Certified defenses aim to provide provable guarantees that a model’s prediction remains stable under bounded perturbations. In the 3D point cloud domain, voting-based methods such as PointGuard~\cite{9577754} and PointCert~\cite{zhang_pointcert_2023} construct multiple sub-clouds via random sampling or hashing, and aggregate their predictions through majority vote to tolerate a limited number of corrupted points. While effective against point additions or deletions, these methods cannot certify robustness under $\ell_2$-bounded perturbations, which preserve the number of points but subtly shift their positions. For example, a typical $\ell_2$ attack perturbs point coordinates along adversarial gradients, altering the geometry without changing the point count~\cite{xiang_generating_2019}. In contrast, randomized smoothing\cite{pmlr-v97-cohen19c} widely effective in 2D image classification—relies on fixed input dimensionality (pixel count), which does not apply to point clouds: even slight coordinate perturbations can disrupt neighborhood relations or trigger preprocessing changes (e.g. sampling or outlier removal), invalidating the theoretical guarantees. Consequently, smoothing‐based approaches fail to reliably certify robustness in point cloud settings.

Spatial-domain point clouds are highly redundant, allowing $\ell_2$ perturbations to distort geometry imperceptibly while preserving point count—challenging existing spatial defenses.Therefore, in this paper, we analyze point cloud robustness in the frequency domain. Prior study\cite{miao_improving_2024} shows that low-frequency components capture semantic structure, while high-frequency components are more vulnerable to adversarial perturbations. We leverage this by applying graph Fourier transform (GFT)\cite{GFT} to project point clouds onto a spectral basis, then we design a frequency-guided sampling strategy called \textit{dense-overlapping spectral window} (d-OSW). Specifically, we first compute the GFT coefficients of each point to identify its dominant spectral response. Based on these responses, we assign points to overlapping frequency bands, ensuring that each slice captures a different portion of the spectrum while retaining sufficient spatial coverage. Each frequency-aware slice is then independently classified using a shared point cloud model. Finally, the overall prediction is determined by majority voting across slices, which improves robustness by reducing the impact of any single corrupted subset.

Theoretically, we establish certified $\ell_2$ robustness guarantees for the final prediction based on the stability of individual slices. Specifically, we derive closed-form expressions for two robustness radii against additive, $\ell_2$-bounded perturbations. The first, a conservative bound, ensures that the prediction remains unchanged as long as all slices are stable. The second, tighter bound leverages the majority voting mechanism: it certifies robustness even when a subset of slices is affected, as long as more than half retain the correct prediction. This analysis connects the spectral structure of the input to provable robustness at the classification level, and represents, to the best of our knowledge, the first certification framework for point clouds under continuous $\ell_2$ perturbations.

We conduct extensive experiments on the ModelNet40 and ScanObjectNN datasets to validate the effectiveness of our method under $\ell_2$ perturbations. Results show that \textbf{FreqCert} significantly improves certified robustness across various backbone networks, especially under strong attacks, demonstrating the effectiveness of our method for robust point cloud recognition.

In summary, our contributions are threefold: (1) We develop a frequency-guided sub-sampling strategy for point clouds based on graph Fourier transform and dense-overlapping spectral windows, and integrate it into \textbf{FreqCert}, a novel certification framework that classifies each spectral slice independently and aggregates predictions via majority voting; (2) We derive closed-form certified $\ell_2$ robustness bounds that explicitly account for the voting mechanism; (3) We conduct extensive experiments on ModelNet40 and ScanObjectNN, demonstrating significant gains in certified accuracy and empirical accuracy under $\ell_2$ perturbations.

\section{Related Work}
\paragraph{Deep Learning on 3D Point Clouds.}
Early works on 3D point cloud learning adapted 2D CNNs via voxelization~\cite{7353481, 7410471, 8374608}, but suffered from sparsity and high computational cost. OctNet~\cite{8100184} and sparse convolutions~\cite{8953494} alleviated these issues through efficient 3D data structures.A major breakthrough came with PointNet~\cite{8099499}, which directly processed raw points using symmetric functions. Its extension PointNet++~\cite{10.5555/3295222.3295263} captured local structures hierarchically. Subsequent architectures such as PointCNN~\cite{li_pointcnn_2018}, KPConv~\cite{9010002}, and DGCNN~\cite{wang_dynamic_2019} improved geometric modeling through learned kernels and graph-based features. GDANet~\cite{xu2021learning} further introduced deformation-aware graph convolution to enhance feature learning under geometric variations. CurveNet~\cite{xiang2021walk} enhances local geometric modeling by leveraging curve-based neighborhoods, leading to improved robustness against fine-grained surface deformations.

Recently, transformer-based models have emerged as state-of-the-art, including Point Transformer~\cite{zhao2021point} and LFT-Net~\cite{gao2022lft}, which integrate self-attention for local-global reasoning. Others, such as PVT~\cite{zhang2022pvt}, combine voxel and point-based transformers for enhanced feature extraction.

\paragraph{Adversarial Attacks \& Defense on 3D Point Clouds.}
Adversarial attacks on 3D point clouds aim to mislead classifiers by manipulating point coordinates. Gradient-based methods such as C\&W and PGD have been adapted to the 3D setting through geometry-aware loss functions~\cite{xiang_generating_2019, wen2020geometry, tsai2020robust}, while black-box approaches leverage generative models for improved transferability~\cite{zhou2020lg}. 
To counter these threats, various defenses have been developed. Filtering-based methods~\cite{zhou_dup-net_2019} remove outlier points, while input transformations and adversarial detectors have also been explored~\cite{dong2020self, liu2019extending}. Purification-based strategies~\cite{wu_if-defense_2021, li2022robust, sun_critical_2023} aim to restore clean geometry before classification. 
Certified defenses provide formal robustness guarantees. Existing efforts include robustness certification under pose variations and rigid transformations~\cite{lorenz2021robustness, perez20223deformrs}, as well as under point-wise and additive perturbations~\cite{9577754, zhang_pointcert_2023}. However, most current methods remain ineffective against subtle yet structured $\ell_2$-bounded attacks.

\section{Method}
The proposed \textbf{FreqCert} is a certified defense framework that leverages frequency-aware subsampling and majority voting to certify the robustness of point cloud classifiers against $\ell_2$-bounded perturbations.In this section, we first introduce the graph Fourier transform (GFT) as the spectral foundation of our method. Then, we present a frequency-guided subsampling strategy that extracts multiple low-frequency sub-point clouds. Finally, we derive a closed-form robustness radius under the $\ell_2$ threat model and majority voting, offering a tight certification guarantee.

\subsection{Graph Fourier transform}
Unlike images that lie on regular grids and can be naturally processed with the discrete Fourier transform (DFT)\cite{rasheed2020image}, 3D point clouds are unordered and reside in irregular metric spaces, which limits the use of standard frequency tools. To address this, we leverage the graph Fourier transform (GFT) \cite{GFT}, which generalizes the concept of Fourier analysis to non-Euclidean domains by treating point clouds as graphs.

\begin{figure}[t]
  \centering
  \includegraphics[width=\linewidth]{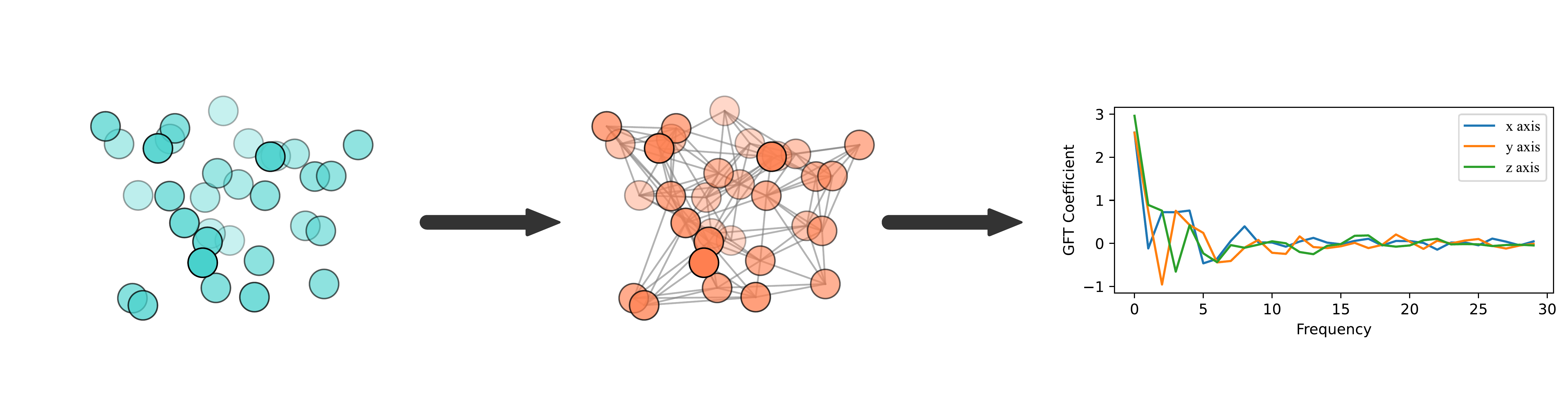}
  \caption{An intuitive visualization of our spectral transformation pipeline. A raw point cloud (left) is first transformed into a $k$-nearest neighbor graph (middle), followed by a graph Fourier transform (right) to derive frequency domain representations.}

  \label{fig:gft_pipeline}
\end{figure}

Let $\mathcal{P} = \{p_i\}_{i=1}^n \subset \mathbb{R}^3$ be a point cloud of $n$ points in Euclidean space. We construct an undirected $k$-nearest neighbor graph $\mathcal{G} = (\mathcal{V}, \mathcal{E})$, where each node $v_i \in \mathcal{V}$ corresponds to a point $p_i \in \mathcal{P}$, and edges $(v_i, v_j) \in \mathcal{E}$ are formed by connecting each point to its $k$ nearest neighbors based on Euclidean distance.

To capture local geometric similarity, edge weights are defined as $w_{ij} = \exp(-\|p_i - p_j\|_2^2)$, forming a weighted adjacency matrix $W$. The degree matrix $D$ is diagonal with entries $D_{ii} = \sum_j W_{ij}$, and the combinatorial graph Laplacian is defined as $L = D - W$.

The Laplacian matrix $L$ is symmetric and positive semi-definite. It can be diagonalized via eigendecomposition as $L = U \Lambda U^\top$, where $U \in \mathbb{R}^{n \times n}$ is an orthonormal matrix of eigenvectors and $\Lambda$ is a diagonal matrix of non-negative eigenvalues. The columns of $U$ form the graph Fourier basis. The entries in $\Lambda$ represent the corresponding graph frequencies, arranged in ascending order from low to high.

Given a graph signal $X \in \mathbb{R}^{n \times c}$, such as the 3D coordinates of the point cloud, the graph Fourier transform is defined as $\widehat{X} = U^\top X$. Each row in $\widehat{X}$ represents the projection of $X$ onto a different frequency component. The inverse transform is given by $X = U \widehat{X}$, allowing full reconstruction from the frequency domain. Figure~\ref{fig:gft_pipeline} visually demonstrates our spectral transformation pipeline, which converts raw point clouds into graph structures and further into frequency representations via GFT.

This frequency decomposition enables a compact and structured analysis of point cloud geometry, which is particularly useful for guiding robust subsampling and perturbation-aware reasoning in downstream tasks.

\subsection{Spectral Subsampling}

To certify robustness against $\ell_2$ perturbations, we must ensure that small changes to the input cannot simultaneously affect all parts of the model. One effective approach is to divide the point cloud into multiple sub-point clouds and base the final decision on a majority vote. If most sub-clouds remain stable under perturbation, the overall prediction remains unchanged.

A key question is how to perform this division. Sampling in the spatial domain, such as by random grouping or clustering in Euclidean space offers little control over how individual point movements affect the slices. A single perturbed point may be assigned to very different slices depending on its shifted location, making the behavior under perturbation difficult to analyze or bound. In addition, spatial closeness alone can be misleading: points that are close in Euclidean space may belong to very different parts of the shape. In contrast, points with similar structural functions may be far apart spatially but share similar spectral responses.

We propose a frequency-guided subsampling algorithm, termed \textbf{dense-overlapping spectral windows (d-OSW)}, to construct robust and analyzable sub-point clouds for certification. Each point in the cloud typically exhibits a dominant response at a particular graph frequency, reflecting its geometric role within the global structure. By identifying these dominant frequencies, we group points based on spectral similarity rather than spatial proximity. To this end, we define a set of overlapping frequency bands and assign each point a soft weight for each band based on the alignment between its dominant frequency and the band center. Sampling is then performed proportionally to these weights, resulting in overlapping sub-point clouds focused on distinct frequency ranges.

We begin by computing the graph Fourier basis of the input point cloud and retaining the first $K$ frequency components. These components are indexed by $0, 1, \dots, K{-}1$. Following ~\citet{miao_improving_2024}, lower frequencies primarily capture global shape information while higher frequencies encode finer geometric details and potential noise. To balance computational efficiency and structural fidelity, we retain only the first $K=128$ components.

The spectral range is evenly divided into $m$ bands. The center of the $b$-th band is given by
\begin{equation}
\mu_b = \left(b + \tfrac{1}{2} \right) \cdot \tfrac{K}{m}, \qquad b = 0, 1, \dots, m{-}1.
\end{equation}

Each point is assigned a dominant frequency index $\nu_i^\star$, defined as the index at which its squared GFT coefficient is maximized. We measure how well this dominant frequency aligns with band $b$ using a Gaussian weight:
\begin{equation}
\gamma_b(i) = \exp\left( -\frac{(\nu_i^\star - \mu_b)^2}{2\sigma^2} \right),
\end{equation}
where $\sigma$ is a bandwidth parameter controlling the overlap between adjacent bands. A smaller $\sigma$ yields sharper, more selective assignments, while a larger $\sigma$ produces smoother, more redundant coverage. Inspired by spectral kernel designs~\cite{hammond2011wavelets,shuman2013emerging}, we set $\sigma = 0.6 \cdot \tfrac{K}{m}$ to ensure effective but not excessive overlap between neighboring bands. This empirically balances slice diversity and stability, and avoids degenerate hard assignments at small $m$.

After computing spectral weights, we normalize them across all points. For each band $b$, we sample exactly $n$ points without replacement from the full point cloud using the normalized weights as sampling probabilities. This yields one sub-point cloud per band.

Repeating the above sampling across all $m$ bands yields $m$ sub-point clouds in total. Since a point may have non-zero weights for multiple bands, it can be included in multiple slices. This overlap introduces controlled redundancy and improves robustness under perturbations.

We refer to this overall procedure and the resulting collection of slices as \textbf{dense-overlapping spectral windows (d-OSW)}. The construction ensures complete spectral coverage while maintaining structural coherence. It also enables precise frequency domain robustness analysis, which we leverage to derive certified $\ell_2$ perturbation bounds.

To ensure consistency between training and certification, we apply the same spectral subsampling strategy during training. Specifically, for each training example, we randomly extract $m$ frequency-guided sub-clouds using d-OSW. The model is trained to correctly classify all sub-clouds. This encourages the network to learn features that are invariant across different spectral slices—an essential property for accurate majority voting at test time. Without such training-time alignment, the model may overfit to global features and fail to generalize when presented with subsampled views during certification.

\begin{algorithm}[tb]
\caption{Dense-Overlapping Spectral Windows (d-OSW)}
\label{alg:dosw}
\textbf{Input}: Point cloud $P = \{p_i\}_{i=1}^N$, GFT basis $\Phi \in \mathbb{R}^{N \times K}$, number of bands $m$, slice size $n$, bandwidth $\sigma$ \\
\textbf{Output}: Set of sub-point clouds $\{S_b\}_{b=0}^{m-1}$

\begin{algorithmic}[1]
\FOR{$i = 1$ to $N$}
    \STATE $\nu_i^\star \leftarrow \arg\max_{\nu} |\Phi_\nu(i)|^2$ \hfill // dominant frequency of point $p_i$
\ENDFOR
\FOR{$b = 0$ to $m-1$}
    \STATE $\mu_b \leftarrow \left(b + \tfrac{1}{2} \right) \cdot \tfrac{K}{m}$ \hfill 
    \FOR{$i = 1$ to $N$}
        \STATE $\gamma_b(i) \leftarrow \exp\left( -\frac{(\nu_i^\star - \mu_b)^2}{2\sigma^2} \right)$ \hfill 
    \ENDFOR
    \STATE Normalize $\{\gamma_b(i)\}_{i=1}^N$ to form a probability distribution
    \STATE $S_b \leftarrow$ sample $n$ points from $P$ without replacement according to weights $\gamma_b$
\ENDFOR
\STATE \textbf{return} $\{S_b\}_{b=0}^{m-1}$
\end{algorithmic}
\end{algorithm}

\subsection{Certified Robustness}
Frequency-aware subsampling offers more than just a way to construct diverse inputs—it enables precise reasoning about how perturbations affect the classifier. Each point is assigned to slices based on its spectral response, and small $\ell_2$ perturbations lead to gradual, predictable changes in this assignment. This continuity allows us to control how input noise propagates through the slicing and voting pipeline.Such analysis is difficult in the spatial domain. When slices are formed by Euclidean grouping or random selection, even minor displacements may abruptly reassign points, making the system sensitive and analytically intractable.

By contrast, the structure imposed by frequency-guided sampling ensures that slice memberships evolve smoothly under perturbation, and redundancy across overlapping slices provides natural resilience. These properties open the door to formal certification.

In the FreqCert framework, robustness is achieved by aggregating the predictions of multiple frequency-aware sub-point clouds. To move beyond empirical robustness and provide formal guarantees, we seek to certify that the final decision remains unchanged under any bounded perturbation of the input.

Let $P = \{p_i\}_{i=1}^N$ be the input point cloud. Using Algorithm~\ref{alg:dosw}, we construct $m$ overlapping frequency slices $\{S_j\}_{j=1}^m$, where each $S_j$ is a subset of points sampled based on spectral similarity. Each slice is passed through a shared base classifier $f$, and the final output is determined by majority voting:
\begin{equation}
    h(P) = \mathrm{MajorityVote}\left( \{ f(S_j) \}_{j=1}^m \right). \label{eq:majority}
\end{equation}

We define certified robustness with respect to the prediction function $h(P)$, which aggregates the results from $m$ frequency-based slices. A prediction is certifiably robust at $P \in \mathbb{R}^{N \times 3}$ if it remains unchanged under any $\ell_2$-bounded perturbation to the input. Specifically, we consider perturbed point clouds $P' = P + \Delta$, where each row $\Delta_i$ satisfies $\|\Delta_i\|_2 \le \varepsilon$. The goal is to find the largest $\varepsilon > 0$ such that
\begin{equation}
    h(P') = h(P), \quad \forall P' \text{ with } \|\Delta_i\|_2 \le \varepsilon \text{ for all } i. \label{eq:certdef}
\end{equation}
This definition corresponds to a per-point $\ell_2$ ball of radius $\varepsilon$ and reflects the total robustness of the system under structured perturbations.

To characterize certified robustness, we analyze how input perturbations influence slice assignments. Each point $p_i$ is grouped based on its dominant graph frequency—the frequency at which its graph Fourier coefficient is maximized.

The sensitivity of this assignment can be quantified by the \textbf{spectral margin} $g_i$, which measures how close point $p_i$ is to changing its slice under perturbation. Let $\nu_i^\star$ be the index of the frequency at which $p_i$ exhibits the largest spectral energy, i.e., the dominant frequency. Then the spectral margin $g_i$ is defined as the minimum distance from $\nu_i^\star$ to the center of any other band:
\begin{equation}
g_i = \min_{b \in \{0, \dots, m-1\}} \left| \nu_i^\star - \mu_b \right|. \label{eq:margin}
\end{equation}
A larger $g_i$ indicates that $p_i$ is farther from band boundaries, making it more robust to perturbations that could alter its slice assignment.

The effect of perturbation on spectral responses is further governed by the graph Laplacian's eigengap $\Delta \lambda$, which controls how sharply frequencies vary with changes in geometry.

During slice construction, each point may softly contribute to multiple bands. Let $\kappa$ be the maximum number of slices a point can appear in; if $t$ points shift their dominant frequencies under perturbation, at most $\kappa t$ slices are affected.

We next provide two formal guarantees that characterize the certified $\ell_2$ robustness of our method.

\begin{theorem}[Certified Perturbation Size for Slice Stability]\label{thm:slice}
Let $P$ be an input point cloud, and suppose the Frobenius norm of the perturbation satisfies
\begin{equation}
    \|\Delta P\|_F < R_{\text{slice}}, \qquad
    R_{\text{slice}} := \frac{g_{\min}\,\Delta\lambda}{4\sqrt{K}}, \label{eq:Rslice}
\end{equation}
where $g_{\min}$ is the minimum spectral margin across all points,
$\Delta\lambda$ is the Laplacian eigengap at index $K$, and $K$ is the number of retained frequencies.

Then no slice changes its prediction, and the overall classification remains unaffected.

Moreover, the bound is tight: any perturbation with norm greater than $R_{\text{slice}}$ may change the predicted label of at least one slice.
\end{theorem}

The result above guarantees robustness by freezing all slice assignments, which can be overly conservative in practice. A more flexible guarantee can be obtained by allowing a limited number of slice changes—as long as the majority vote remains unaffected. The following result formalizes this relaxed but stronger form of certified robustness.

\begin{theorem}[Certified Perturbation Size for $\ell_2$ Robustness]\label{thm:majority}
Let \( g_{(1)} \le \cdots \le g_{(N)} \) be the sorted spectral margins of the points. Define
\[
\alpha = \left\lfloor \frac{m - 1}{2} \right\rfloor, \quad p = \left\lfloor \frac{\alpha}{\kappa} \right\rfloor.
\]
Then for
\begin{equation}
    R^\star := \frac{\Delta \lambda\,\sqrt{K}}{8}
    \left( \sum_{i=1}^{p+1} g_{(i)}^2 \right)^{1/2},                                    \label{eq:Rmajority}
\end{equation}

any perturbation \( \Delta P \) with \( \| \Delta P \|_F < R^\star \) affects at most \( p \) points, and thus no more than \( \kappa p \le \alpha \) slices. The majority vote, and hence the prediction \( h(P) \), remains unchanged.

This radius is tight: any uniform improvement of \( R^\star \) would fail in the worst case without additional assumptions.
\end{theorem}

\begin{proof}
See Appendix for the proofs of Theorem 1 and Theorem 2.
\end{proof}

Eqs.\eqref{eq:Rmajority} reveal that the certified radius depends on three key factors: the spectral margins of individual points, the Laplacian eigengap $\Delta\lambda$, and the overlap multiplicity $\kappa$. Larger margins and a wider eigengap strengthen robustness, while a lower $\kappa$ reduces the number of slices affected by each perturbation. These insights suggest that designing d-OSW to amplify spectral separation and limit redundancy—without sacrificing coverage—can directly improve the certified $\ell_2$ radius.

\section{Experiments}
\subsection{Experimental Setup}
\subsubsection{Datasets and moudles}
We conduct experiments on two widely used datasets: \textbf{ModelNet40}~\cite{wu20153d} and \textbf{ScanObjectNN}~\cite{uy-scanobjectnn-iccv19}. ModelNet40 consists of 12,311 synthetic 3D CAD models from 40 object categories, with 9,843 samples for training and 2,468 for testing. Each object is uniformly sampled into 1,024 surface points.
To assess robustness under more realistic conditions, we also evaluate on ScanObjectNN, which contains 2,902 real-world scanned objects from 15 categories. We adopt the PB\_T50\_RS variant, which includes background clutter and partial occlusions. Each object is represented by 2,048 points, and the dataset is split into 2,048 training and 881 test samples.
We use PointNet~\cite{8099499} and DGCNN~\cite{10.1145/3326362} as representative classification backbones due to their widespread use in point cloud recognition tasks.

\subsubsection{Compared methods}
We compare FreqCert with two baselines. The undefended classifier refers to the standard point cloud model without any robustness enhancement, trained on clean data. We evaluate it to assess the impact of our certification strategy. In addition, we include Randomized Smoothing (RS)~\cite{pmlr-v97-cohen19c}, a widely used certified defense for $\ell_2$ perturbations. RS adds isotropic Gaussian noise to the input and performs majority voting over multiple predictions. It provides a certified radius within which the classifier's prediction remains provably unchanged.

\subsubsection{Evaluation Metrics}

We evaluate robustness from both certified and empirical perspectives. \textit{Certified accuracy} is the proportion of test samples that are correctly classified and provably robust against all $\ell_2$ perturbations within a radius $\epsilon$, providing a worst-case guarantee. In contrast, \textit{empirical accuracy} reflects the proportion of samples that remain correctly classified under specific adversarial attacks of strength $\epsilon$. As it accounts for all possible perturbations, certified accuracy is always a conservative lower bound on empirical accuracy at the same $\epsilon$.

In our experiments, empirical robustness is evaluated using adversarial examples generated by projected gradient descent (PGD)~\cite{xiang_generating_2019}, a widely adopted attack in the point cloud literature. Given a clean input point cloud $\mathbf{X}_o \in \mathbb{R}^{N \times 3}$, consisting of $N$ points in 3D space, we generate a perturbed version $\mathbf{X}_o^{\text{adv}}$ such that the total displacement is constrained by $\|\mathbf{X}_o^{\text{adv}} - \mathbf{X}_o\|_2 \leq \epsilon$, where $\epsilon$ controls the attack strength. We evaluate robustness across a range of $\epsilon$ values and report both certified and empirical accuracy as functions of $\epsilon$ in the final robustness plots.

\subsubsection{Parameter setting}
Main experiments use DGCNN~\cite{10.1145/3326362} as the backbone, with a 20-nearest neighbor graph, trained for 250 epochs using the Adam optimizer with a batch size of 32 and an initial learning rate of 0.001. For FreqCert, each point cloud is decomposed into \(m=32\) frequency-guided sub-clouds, and each sub-cloud contains \(n=128\) points. These values are chosen to balance spectral coverage and certification tightness. In particular, \(n\) must not be too small under \(\ell_2\)-norm perturbations, since low-resolution sub-clouds are more sensitive to global geometric distortions, which degrades prediction stability and reduces certified accuracy. The same subsampling strategy is applied during training to ensure consistency between learning and certification. For the Randomized Smoothing baseline, we adopt $\sigma = 0.5$, which provides a reasonable trade-off between clean accuracy and certified robustness.

\subsection{Experimental Results}
\subsubsection{Main results:}
As shown in Fig.~\ref{fig:cert_emp}, we compare the certified and empirical accuracy of FreqCert and randomized smoothing under $\ell_2$-norm bounded point perturbations on ModelNet40. At $\epsilon = 0$, the undefended model achieves the highest empirical accuracy, but its performance quickly drops as $\epsilon$ increases, revealing its vulnerability to structured input distortions.

\begin{figure}[t]
  \centering
  \begin{subfigure}[b]{0.48\linewidth}
    \includegraphics[width=\linewidth]{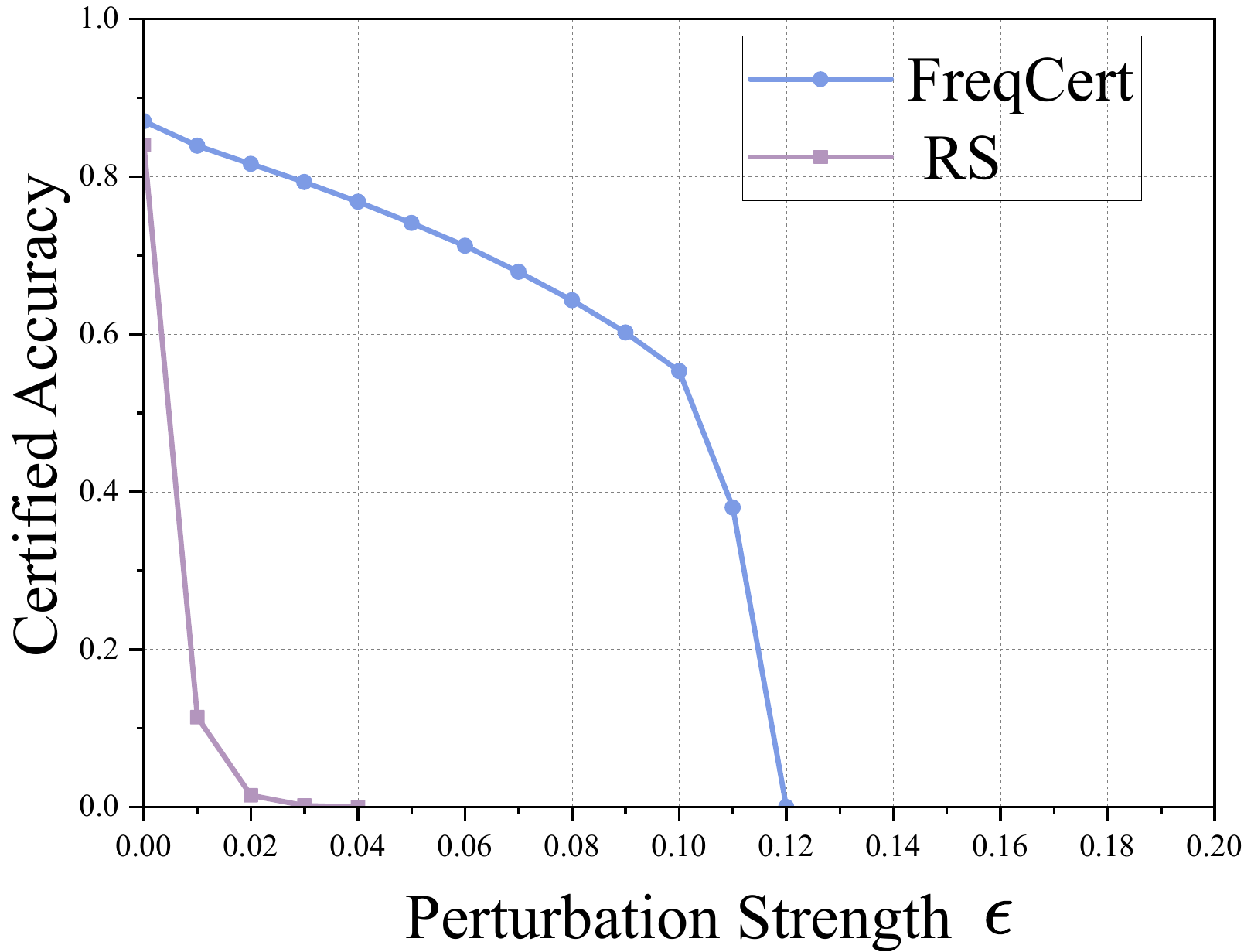}
    \caption{}
    \label{fig:cert}
  \end{subfigure}
  \begin{subfigure}[b]{0.48\linewidth}
    \includegraphics[width=\linewidth]{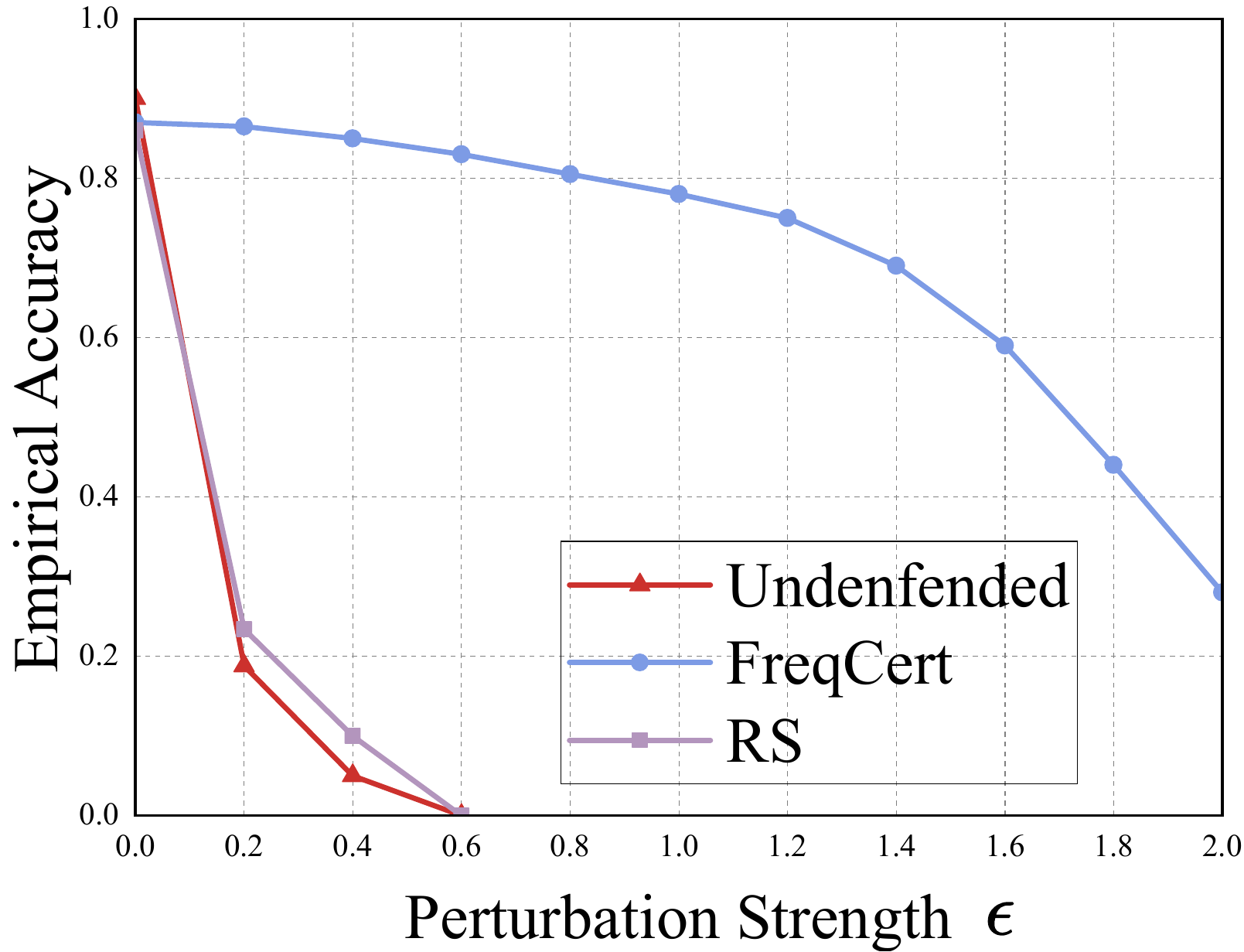}
    \caption{}
    \label{fig:emp}
  \end{subfigure}
  \caption{
    Certified accuracy (a) and empirical accuracy (b) on the ModelNet40 dataset under increasing $\ell_2$ perturbation strength $\epsilon$.
    FreqCert consistently achieves higher robustness than randomized smoothing across all tested perturbation levels.
    Results on the ScanObjectNN dataset are provided in the appendix.
  }
  \label{fig:cert_emp}
\end{figure}

In terms of certified robustness, FreqCert demonstrates a clear advantage. It maintains non-trivial certified accuracy up to \( \epsilon = 0.12 \), whereas randomized smoothing fails to certify any samples beyond \( \epsilon = 0.02 \). The certified accuracy curve of FreqCert decays smoothly, with a noticeable drop-off near \( \epsilon = 0.10 \), indicating that its robustness degrades in a controlled manner as perturbations increase. In contrast, the certification curve of randomized smoothing remains nearly flat and close to zero throughout, reflecting its poor ability to provide formal guarantees on 3D data.

\begin{figure*}[t]
  \centering
  \begin{subfigure}[b]{0.32\linewidth}
    \includegraphics[width=\linewidth]{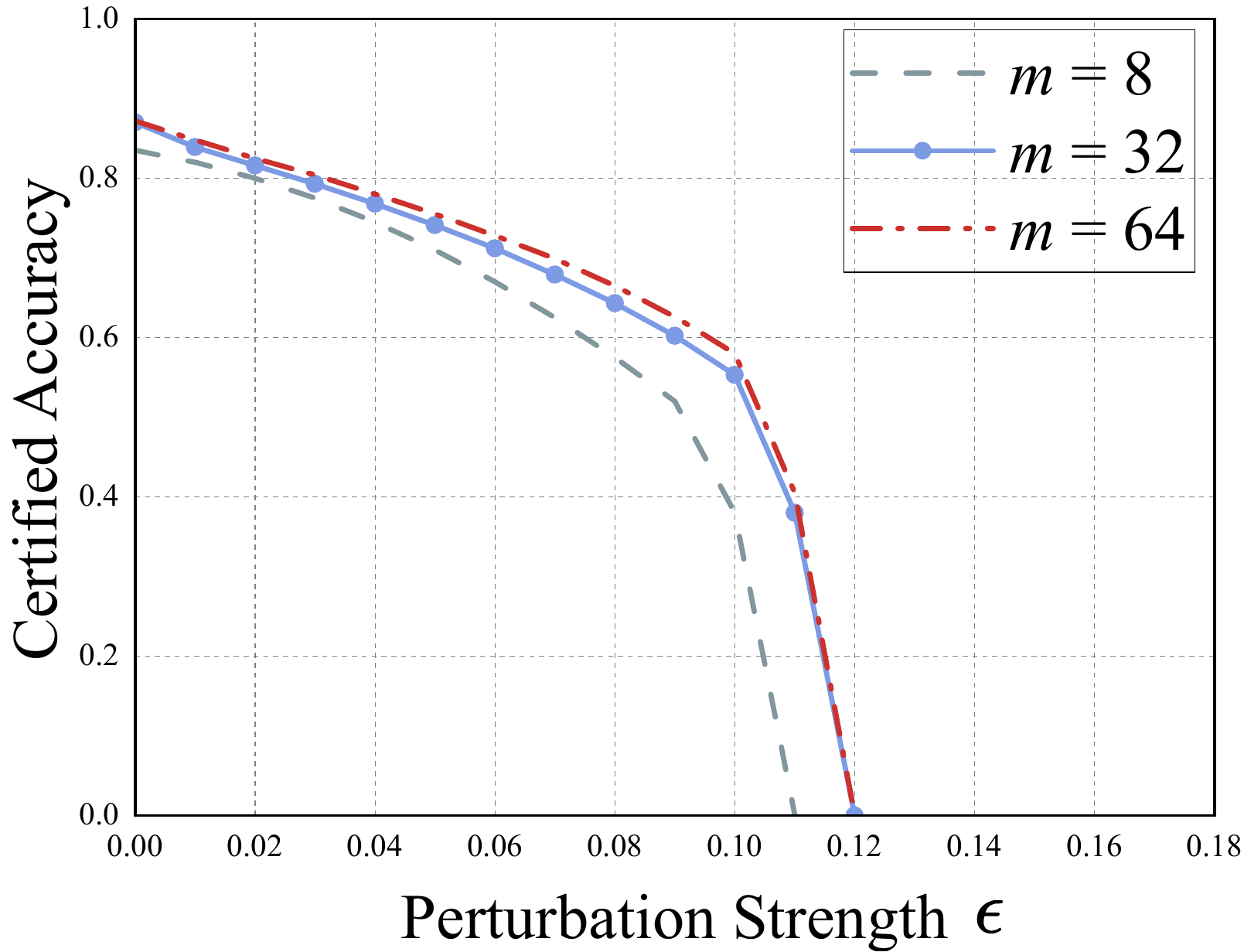}
    \caption{}
    \label{fig:4a}
  \end{subfigure}
  \hfill
  \begin{subfigure}[b]{0.32\linewidth}
    \includegraphics[width=\linewidth]{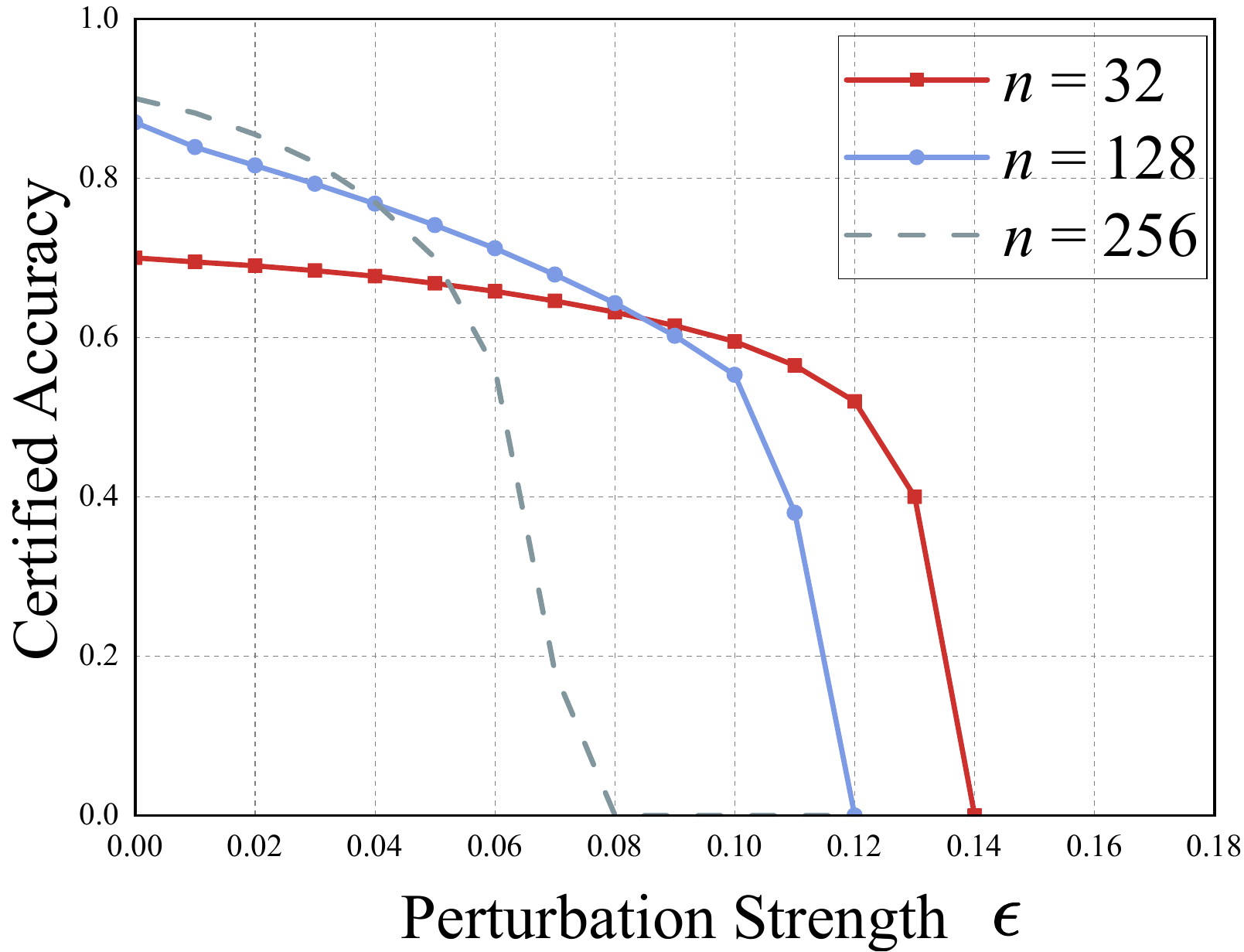}
    \caption{}
    \label{fig:4b}
  \end{subfigure}
  \hfill
  \begin{subfigure}[b]{0.32\linewidth}
    \includegraphics[width=\linewidth]{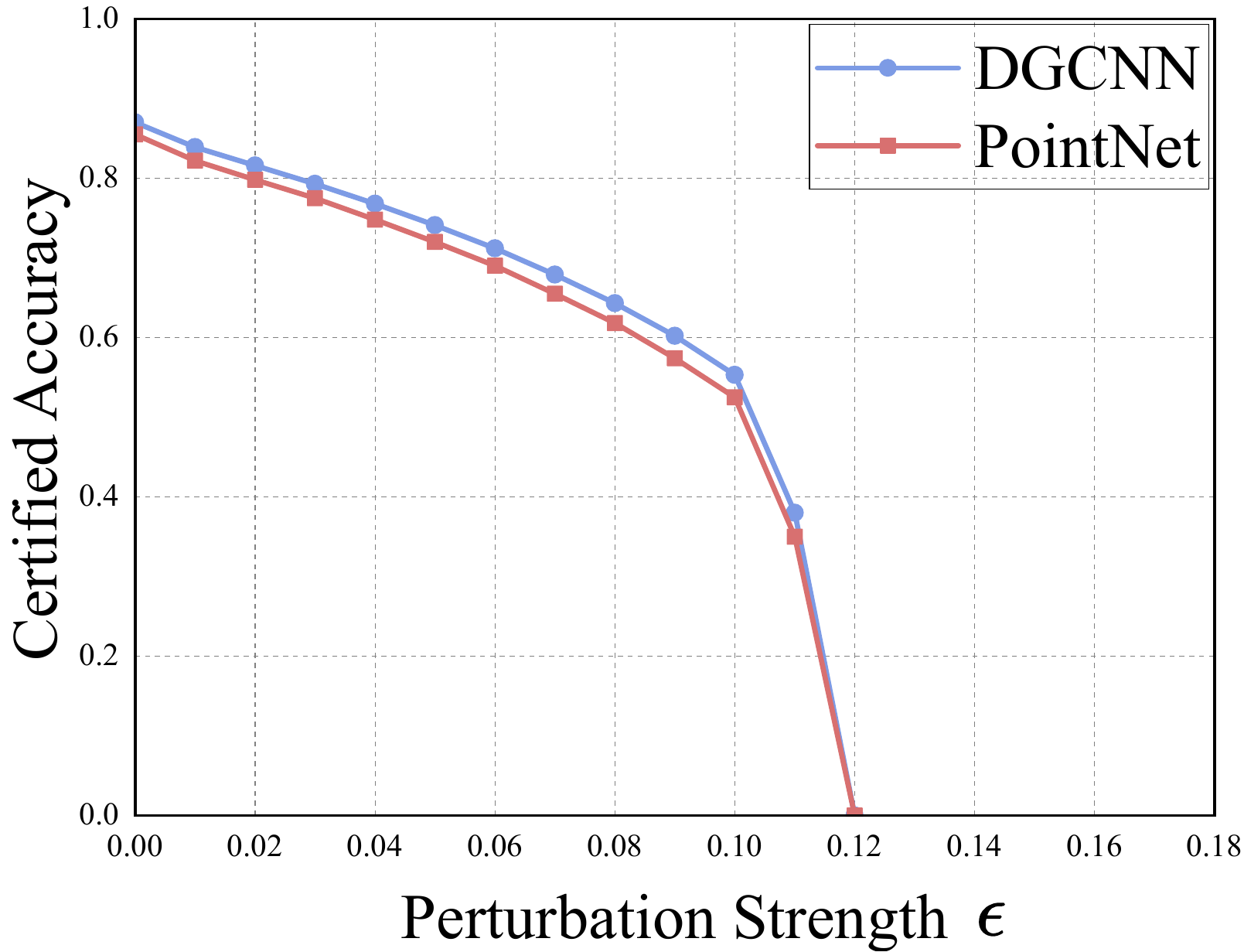}
    \caption{}
    \label{fig:4c}
  \end{subfigure}
  \caption{
    Impact of key factors on certified accuracy.
    (a) Impact of  $m$ on FreqCert.
    (b) Impact of  $n$ on FreqCert.
    (c) Comparing certified accuracy of FreqCert across different backbone architectures
  }
  \label{fig:param_cert}
\end{figure*}

The sharp contrast between the two curves highlights a key limitation of additive-noise-based methods when applied to point clouds: even weak Gaussian perturbations can disrupt the underlying geometry, breaking the conditions required for certification. In contrast, FreqCert leverages frequency-domain decomposition and structured voting to certify a larger fraction of test samples across all tested radii. The smooth decay of its certification curve suggests that the method is inherently more resilient to localized distortions and better captures global structural properties.

On the empirical side, FreqCert maintains high robustness against PGD attacks, preserving over 70\% accuracy at $\epsilon = 1.0$ and retaining non-trivial performance even at $\epsilon = 2.0$. Compared to randomized smoothing, which collapses rapidly and becomes ineffective beyond $\epsilon = 0.4$, FreqCert degrades more gradually. This behavior reflects the advantage of frequency-based subsampling: by partitioning the input into multiple overlapping spectral slices, localized perturbations are prevented from dominating the prediction, and the majority vote mechanism mitigates adversarial impact across slices.

\subsubsection{Comparison with Empirical Defenses:}
To further evaluate the empirical robustness of FreqCert, we compare it with several representative empirical defenses under strong adversarial attacks, including DUP-Net~\cite{zhou_dup-net_2019} and IF-Defense~\cite{wu_if-defense_2021}. DUP-Net enhances robustness through outlier removal and upsampling, while IF-Defense projects the input onto an implicit surface to mitigate adversarial distortions.

In addition to the standard PGD attack, we also include AdvPC~\cite{WOS:001500580900015}, a point-cloud-specific method that generates semantically meaningful perturbations by approximating latent geometric deformations.

All defenses are evaluated under a fixed perturbation strength of $\epsilon = 1.25$, defined as the $\ell_2$ norm of the total displacement applied to the point cloud. For fair comparison, all methods use the same set of backbones: PointNet~\cite{8099499},  DGCNN~\cite{10.1145/3326362}, and CurveNet~\cite{xiang2021walk}.

\begin{table}[t]
\centering
\small
\begin{tabular}{l lccc}
\toprule
\multicolumn{2}{c}{} & PointNet & DGCNN & CurveNet \\
\midrule
\multirow{2}{*}{FreqCert}   & PGD    & \textbf{78.9} & \textbf{74.8} & \textbf{70.5} \\
                            & AdvPC  & \textbf{69.4} & \textbf{71.7} & \textbf{73.3} \\
\midrule
\multirow{2}{*}{IF-Defense} & PGD    & 74.1 & 70.6 & 69.6 \\
                            & AdvPC  & 69.3 & 64.7 & 57.9 \\
\midrule
\multirow{2}{*}{DUP-Net}    & PGD    & 74.8 & 68.6 & 69.2 \\
                            & AdvPC  & 68.9 & 67.0 & 68.4 \\
\bottomrule
\end{tabular}
\caption{Empirical accuracy (\%) under PGD and AdvPC attacks with $\ell_2$ perturbation of strength $\epsilon = 1.25$.
}
\label{tab:empirical_def}
\end{table}

As shown in Table~\ref{tab:empirical_def}, FreqCert consistently achieves the highest empirical accuracy across all three backbones under both PGD and AdvPC attacks. The performance gap is especially pronounced under AdvPC, where IF-Defense suffers notable drops (e.g., $57.9\%$ on CurveNet), while FreqCert maintains robust performance (e.g., $73.3\%$). These results indicate that FreqCert generalizes better to strong, shape-aware attacks compared to existing empirical defenses.

\subsubsection{Impact of $m$:}
 As shown in Fig.~\ref{fig:param_cert}(a), the $m{=}32$ curve consistently outperforms $m{=}8$ in the mid-range region \( \epsilon \in [0.06, 0.10] \), where the separation is most pronounced. For \( \epsilon < 0.04 \), all three curves are nearly identical. The $m{=}8$ configuration drops sharply near \( \epsilon \approx 0.10 \), while $m{=}32$ and $m{=}64$ remain valid up to \( \epsilon \approx 0.12 \). Between them, $m{=}64$ shows a slight but marginal advantage. This reflects a trade-off: increasing \( m \) improves voting robustness by raising \( \alpha = \lfloor(m{-}1)/2\rfloor \) and \( q = \lfloor\alpha/\kappa\rfloor + 1 \), but also increases the overlap factor \( \kappa \), limiting further gains.

\subsubsection{Impact of $n$:}
 As shown in Fig.~\ref{fig:param_cert}(b), larger slices (\(n{=}256\)) achieve the highest certified accuracy at small perturbations but degrade rapidly, with the certificate vanishing around \(\epsilon\!\approx\!0.08\). In contrast, smaller slices (\(n{=}32\)) maintain non-zero certified accuracy up to \(\epsilon\!\approx\!0.14\), albeit with significantly lower performance in the low-perturbation regime. This trade-off arises because large slices capture more global structure but are more vulnerable to global geometric distortions, while small slices are more robust to such perturbations but suffer from higher prediction variance due to limited resolution.

\subsubsection{Backbone comparison:} 
 Fig.~\ref{fig:param_cert}(c) shows that DGCNN stays consistently above PointNet across the whole perturbation range (\(\epsilon\le 0.10\)), but the gap is small (typically below 2\%). Both curves drop abruptly and reach zero around \(\epsilon\!\approx\!0.12\). This indicates that FreqCert is largely architecture-agnostic—the certificate is dominated by the spectral subsampling and voting scheme—while the stronger local feature extraction of DGCNN provides a slight advantage.

\section{Conclusion}
In this work, we propose FreqCert, a certification framework based on frequency-domain subsampling for 3D point clouds. A novel spectral slicing algorithm enables closed-form $\ell_2$ robustness certificates against structured perturbations without adversarial training or randomized smoothing. Theoretical analysis proves tightness under minimal assumptions, and experiments on ModelNet40 and ScanObjectNN demonstrate consistent improvements in certified and empirical robustness.

\bibliography{aaai2026}
\clearpage


\appendix
\section{Appendix}
\subsection{Proof of Theorem~\ref{thm:slice}}
\label{app:proof_slice}

We restate the certified slice stability guarantee as follows. Let $P$ be an input point cloud, and suppose the Frobenius norm of the perturbation satisfies
\begin{equation}
    \|\Delta P\|_F < R_{\text{slice}}, \qquad
    R_{\text{slice}} := \frac{g_{\min}\,\Delta\lambda}{4\sqrt{K}}, \label{eq:Rslice2}
\end{equation}
where $g_{\min}$ is the minimum spectral margin across all points,
$\Delta\lambda$ is the Laplacian eigengap at index $K$, and $K$ is the number of retained frequencies.

Then no slice changes its prediction, and the overall classification remains unaffected.

Moreover, the bound is tight: any perturbation with norm greater than $R_{\text{slice}}$ may change the predicted label of at least one slice.

\vspace{0.5em}
We aim to show that small perturbations do not alter slice assignments. The key idea is to ensure the dominant frequency of each point remains unchanged. This requires the perturbation-induced change in spectral energy to be smaller than the gap between the dominant and other frequencies. We first bound this spectral gap.
\paragraph{Energy gap bound via spectral margin}
Let \( \nu_i^\star := \arg\max_{\nu} |u_\nu^\top p_i|^2 \) be the dominant frequency of point \( p_i \), and let \( g_i \) be its distance to the nearest band center:
\[
g_i := \min_b \left| \nu_i^\star - \mu_b \right|, \quad \text{with } \mu_b = \left(b + \tfrac{1}{2} \right)\tfrac{K}{m}.
\]
Define the minimum energy gap between the dominant and any other frequency as
\begin{equation}
\delta_i := \min_{\nu \ne \nu_i^\star} \left( |u_{\nu_i^\star}^\top p_i|^2 - |u_\nu^\top p_i|^2 \right). \label{eq:delta_def}
\end{equation}
Since the spectral energies \( |u_\nu^\top p_i|^2 \) sum to \( \|p_i\|_2^2 \le 1 \), we have \( \max_\nu |u_\nu^\top p_i|^2 \ge 1/K \), and the difference in energy between adjacent frequencies is bounded. Using a discrete derivative approximation and the Courant–Fischer relation, we obtain:
\begin{equation}
\delta_i \ge \frac{g_i\,\Delta\lambda}{4\sqrt{K}}. \label{eq:gap_bound}
\end{equation}

\paragraph{Perturbation-induced spectral energy variation}
For any frequency \( \nu \), the variation in spectral energy caused by perturbation \( \Delta p_i \) is
\begin{equation}
\left|\, |u_\nu^\top(p_i + \Delta p_i)|^2 - |u_\nu^\top p_i|^2 \,\right| \le 2\|\Delta p_i\|_2 + \|\Delta p_i\|_2^2. \label{eq:spectral_variation}
\end{equation}
This follows from the expansion \( (a + b)^2 - a^2 = 2ab + b^2 \) with \( a = u_\nu^\top p_i \), \( b = u_\nu^\top \Delta p_i \), and the Cauchy–Schwarz inequality.

\paragraph{Pointwise slice stability}
Suppose
\begin{equation}
\|\Delta p_i\|_2 < \frac{g_i\,\Delta\lambda}{4\sqrt{K}}. \label{eq:pointwise_margin}
\end{equation}
Then from Eq.~\eqref{eq:spectral_variation},
\begin{equation}
2\|\Delta p_i\|_2 + \|\Delta p_i\|_2^2 < \frac{g_i\,\Delta\lambda}{2\sqrt{K}} \le \delta_i,
\label{eq:energy_margin}
\end{equation}
since \( 2x + x^2 < x/2 \) when \( x < 1/4 \). This implies that the ordering of spectral energies remains unchanged, so the dominant frequency \( \nu_i^\star \) does not shift, and the point \( p_i \) retains its original slice assignment.

\paragraph{Global guarantee}
If the Frobenius norm satisfies \( \|\Delta P\|_F < R_{\text{slice}} \), then for every \( i \),
\begin{equation}
\|\Delta p_i\|_2 \le \|\Delta P\|_F < \frac{g_{\min}\,\Delta\lambda}{4\sqrt{K}} \le \frac{g_i\,\Delta\lambda}{4\sqrt{K}},
\label{eq:pointwise_bound}
\end{equation}
so all points satisfy the pointwise condition in Eq.~\eqref{eq:pointwise_margin}. Thus, no slice assignment changes, and the prediction is stable:
\begin{equation}
h(P + \Delta P) = h(P). \label{eq:stable_prediction}
\end{equation}

\paragraph{Tightness}
Select the point \( p_j \) with \( g_j = g_{\min} \) and construct a perturbation
\begin{equation}
\Delta p_j = (R_{\text{slice}} + \varepsilon)\, \mathbf{v}, \label{eq:tight_vec}
\end{equation}
where \( \mathbf{v} \) is a unit vector chosen to maximally reduce the spectral energy gap, e.g., in the direction \( u_{\tilde{\nu}} - u_{\nu_j^\star} \) with \( |\tilde{\nu} - \nu_j^\star| = g_{\min} \). All other points are fixed, so \( \|\Delta P\|_F = R_{\text{slice}} + \varepsilon \) and
\begin{equation}
\|\Delta p_j\|_2 > \frac{g_j\,\Delta\lambda}{4\sqrt{K}},
\label{eq:flip_threshold}
\end{equation}
violating the pointwise condition. The dominant frequency of \( p_j \) can then change, altering its slice assignment and possibly flipping the final prediction if the voting margin is 1. Therefore, no uniform increase of \( R_{\text{slice}} \) is possible. \qed




.

\subsection{Proof of Theorem~\ref{thm:majority}}
Let \( g_{(1)} \le \cdots \le g_{(N)} \) be the sorted spectral margins of the points. Define
\[
\alpha = \left\lfloor \frac{m - 1}{2} \right\rfloor, \quad p = \left\lfloor \frac{\alpha}{\kappa} \right\rfloor.
\]
Then for
\begin{equation}
    R^\star := \frac{\Delta \lambda\,\sqrt{K}}{8}
    \left( \sum_{i=1}^{p+1} g_{(i)}^2 \right)^{1/2},                                    \label{eq:Rmajority2}
\end{equation}

any perturbation \( \Delta P \) with \( \| \Delta P \|_F < R^\star \) affects at most \( p \) points, and thus no more than \( \kappa p \le \alpha \) slices. The majority vote, and hence the prediction \( h(P) \), remains unchanged.

This radius is tight: any uniform improvement of \( R^\star \) would fail in the worst case without additional assumptions.

\vspace{1em}

We now provide a detailed proof, structured in four parts. We begin by characterizing when an individual point’s slice assignment may change under perturbation. We then upper bound the total number of affected points, followed by an analysis of their cumulative impact on slice-level predictions. Finally, we establish the tightness of the certified radius.

\paragraph{Slice Flips Induced by Point Perturbation}
As shown in Theorem~\ref{thm:slice}, a point \( p_i \) will retain its slice assignment as long as
\begin{equation}
\| \Delta p_i \|_2 < \frac{g_i \, \Delta\lambda}{4\sqrt{K}}.
\label{eq:thm2_point_margin}
\end{equation}
This gives a \emph{per-point margin} for stability. Conversely, if this inequality is violated, the slice assignment of \( p_i \) may change.

Let \( \mathcal{I}_{\text{flip}} \subseteq \{1,\dots,N\} \) denote the set of perturbed points whose slice assignment may change. Then the total number of potentially changed slices is at most
\begin{equation}
|\mathcal{S}_{\text{flip}}| \le \kappa \cdot |\mathcal{I}_{\text{flip}}|.
\label{eq:thm2_slice_count}
\end{equation}

\paragraph{Bounding the Number of Flipped Slices}
Suppose \( \Delta P \) satisfies Eq.~\eqref{eq:Rmajority2}. Then by the definition of Frobenius norm:
\begin{equation}
\sum_{i=1}^{N} \|\Delta p_i\|_2^2 < \left( \frac{\Delta\lambda\,\sqrt{K}}{8} \right)^2 \cdot \sum_{i=1}^{p+1} g_{(i)}^2.
\label{eq:thm2_fro_bound}
\end{equation}
If more than \( p \) points violate their individual margin thresholds, then for at least \( p+1 \) indices \( i \),
\begin{equation}
\| \Delta p_i \|_2^2 \ge \left( \frac{g_i \, \Delta\lambda}{4\sqrt{K}} \right)^2.
\label{eq:thm2_point_violate}
\end{equation}
Summing over those indices:
\begin{equation}
\sum_{i=1}^{N} \|\Delta p_i\|_2^2 \ge \left( \frac{\Delta\lambda\,\sqrt{K}}{4} \right)^2 \cdot \sum_{i=1}^{p+1} g_{(i)}^2,
\label{eq:thm2_contra}
\end{equation}
which contradicts the assumed bound in Eq.~\eqref{eq:thm2_fro_bound}. Therefore, at most \( p \) points can change their slice assignment.

\paragraph{Stability of the Final Prediction}
Even if up to \( p \) points change their slice assignment, the total number of affected slices is at most
\begin{equation}
\kappa p \le \alpha.
\label{eq:thm2_slice_majority}
\end{equation}
Since the total number of slices is \( m \), and a majority requires at least \( \alpha + 1 \) consistent votes, the final prediction is preserved as long as no more than \( \alpha \) slices are perturbed. Thus,
\begin{equation}
h(P + \Delta P) = h(P).
\label{eq:thm2_final_pred}
\end{equation}

\paragraph{Tightness of the Bound}
Choose the \(q=p+1\) points with the smallest margins \(g_{(1)},\dots,g_{(q)}\) and leave all other points unperturbed.
For each selected index \(i\) take the worst-case unit vector \(\boldsymbol{v}_i\) that maximally decreases the dominant–runner-up energy gap, and set
\begin{equation}
\Delta p_i = \Bigl(\tfrac{g_i\,\Delta\lambda}{4\sqrt{K}}\Bigr)\,
             \boldsymbol{v}_i ,\qquad i=1,\dots,q,
\label{eq:thm2_tight_vec}
\end{equation}
with \(\Delta p_j=0\) for \(j\notin\{1,\dots,q\}\).
Assume the \(\boldsymbol{v}_i\) are mutually orthogonal (this can always be arranged in \(\mathbb{R}^3\) by placing the \(q\) points far apart). We have
\begin{equation}
\begin{split}
\|\Delta P\|_F^2
  &= \sum_{i=1}^{q}
     \left(\tfrac{g_{(i)}\,\Delta\lambda}{4\sqrt{K}}\right)^2 \\
  &= \left(\tfrac{\Delta\lambda\,\sqrt{K}}{8}\right)^2
     \sum_{i=1}^{q} g_{(i)}^2
  = (R^\star)^2.
\end{split}
\label{eq:thm2_tight_exact}
\end{equation}

Thus every chosen point is perturbed exactly to its flip threshold and all
\(q\) slice assignments can change.

\textbf{Any radius larger than \(R^\star\) fails.}
Fix an arbitrary \(\varepsilon>0\) and increase each \(\Delta p_i\) in
\eqref{eq:thm2_tight_vec} by the factor \(\bigl(1+\tfrac{\varepsilon}{qR^\star}\bigr)\).
Because the directions are orthogonal,
\[
\|\Delta P\|_F = R^\star + \varepsilon .
\]
The same \(q\) points still flip, so at least \(\kappa q\) slices change.  Since
\(q=p+1\) and \(p=\lfloor \alpha/\kappa \rfloor\), we have
\(\kappa q \ge \alpha + 1\); the ensemble therefore loses its majority and the
prediction can differ.  Consequently, any uniform increase of the certified
radius beyond \(R^\star\) admits a counter-example, proving tightness. \qed

\newpage
\subsection{Results on ScanObjectNN}

\begin{figure}[t]
  \centering
  \begin{subfigure}[b]{0.48\linewidth}
    \includegraphics[width=\linewidth]{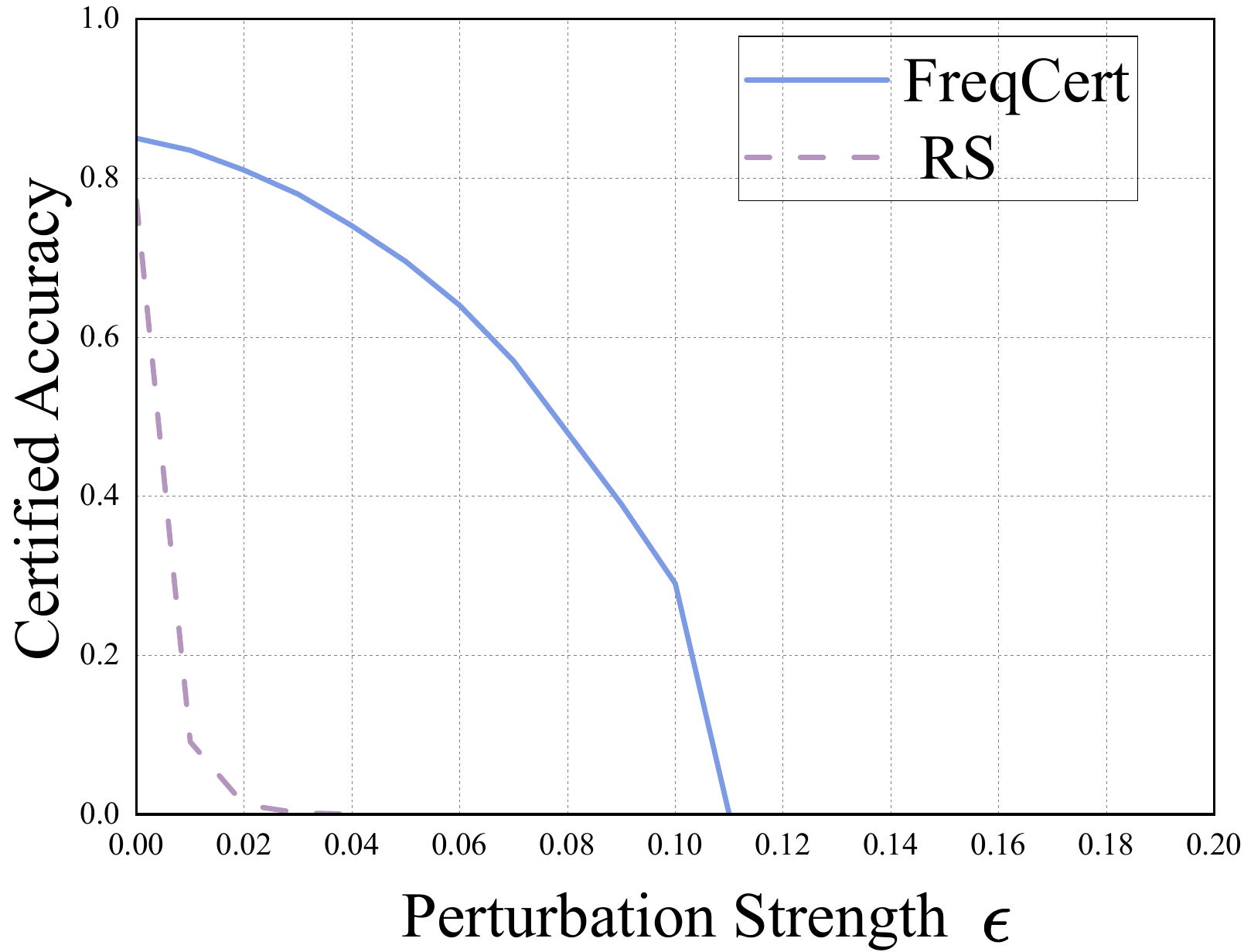}
    \caption{}
    \label{fig:cert_scanobject}
  \end{subfigure}
  \begin{subfigure}[b]{0.48\linewidth}
    \includegraphics[width=\linewidth]{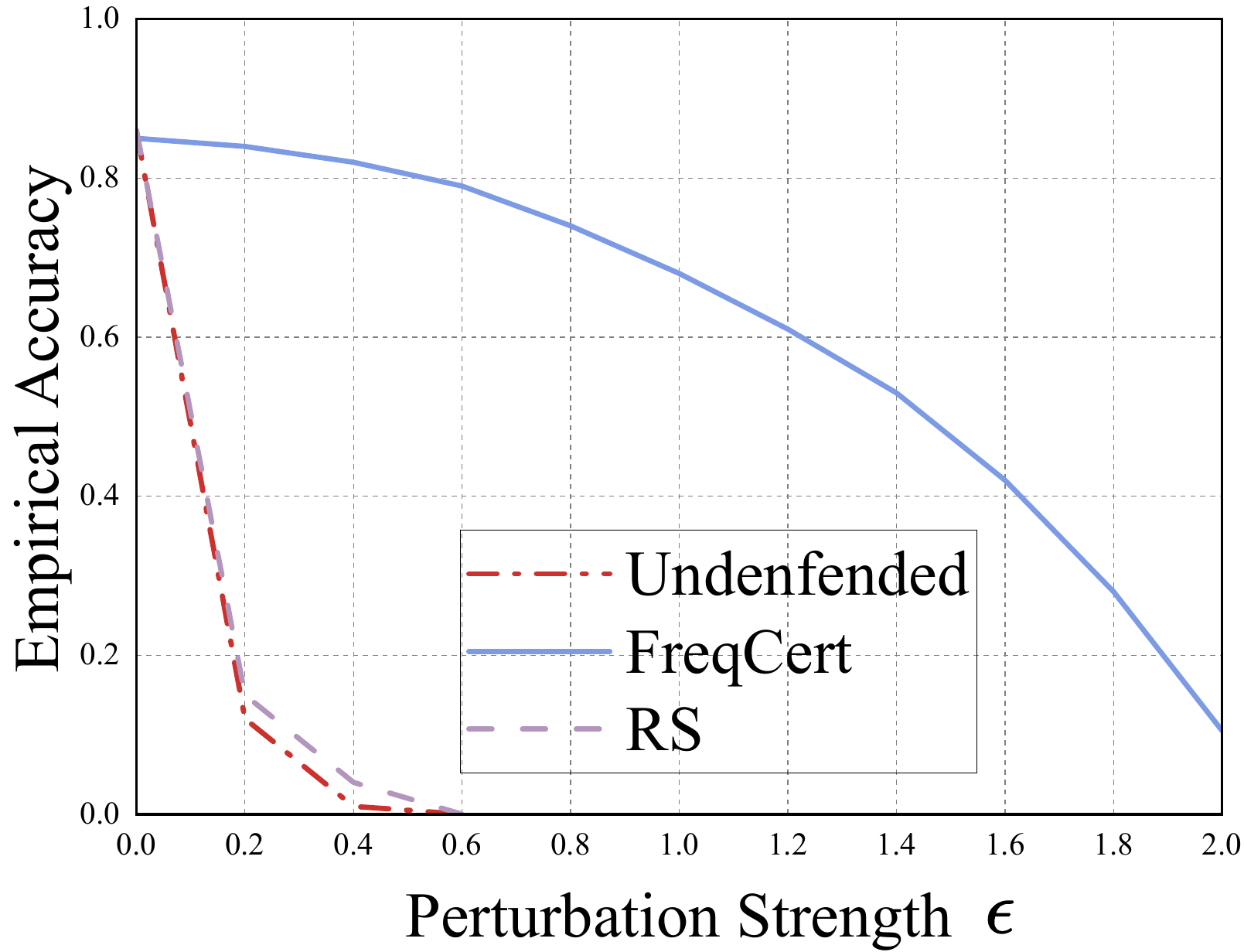}
    \caption{}
    \label{fig:emp_scanobject}
  \end{subfigure}
   \caption{
    Robustness comparison on the Robustness comparison on the ScanObjectNN-OBJ\_Only dataset. 
    (a) Certified accuracy under different $\ell_2$ perturbation strengths. 
    (b) Empirical accuracy against PGD attacks with increasing $\ell_2$ strength. 
  }
  \label{fig:scanobjectnn_results}
\end{figure}

The results shown in Fig.~\ref{fig:scanobjectnn_results} demonstrate the robustness performance of different methods on the ScanObjectNN dataset. FreqCert consistently outperforms Randomized Smoothing in both certified and empirical accuracy across all perturbation levels on ScanObjectNN. The advantage is particularly pronounced in the certified setting, where FreqCert maintains robustness well beyond the breakdown point of RS. These results highlight the effectiveness of frequency-guided slicing in capturing structural features under complex perturbations.

Compared to results on ModelNet40, the robustness curves on ScanObjectNN exhibit a steeper decline. This is primarily due to the presence of real-world artifacts such as noise, occlusion, and misalignment, which increase spectral instability and challenge slice consistency. Despite this, FreqCert remains substantially more robust than baseline methods, demonstrating its adaptability across diverse data conditions.
\end{document}